\pdfoutput=1
\documentclass[10pt]{scrartcl}

\usepackage[english]{babel} 
\usepackage[protrusion=true,expansion=true]{microtype} 
\usepackage{amsmath,amsfonts,amsthm,amssymb,color,epsfig,fullpage} 
\usepackage{verbatim,fancyvrb}
\usepackage{alltt}
\usepackage{tikz}
\usepackage{enumerate}
\usepackage{caption}
\usepackage{subfigure}
\usepackage[]{algorithm}
\usepackage{algpseudocode}
\usepackage{stmaryrd}
\usepackage{afterpage}
\usepackage{url}
\usepackage[english=american]{csquotes}

\newtheorem{definition}{Definition}

\newtheorem{proposition}[definition]{Proposition}

\def\N{{\mathbb N}}

\def\R{{\mathbb R}}

\DeclareMathOperator*{\argmin}{argmin}
\DeclareMathOperator{\dtw}{\delta}

\newcommand{\commentout}[1]{}

\newcommand{\abs}[1]{\mathop{\left\lvert #1 \right\rvert}} 
\newcommand{\args}[1]{\mathop{\left( #1 \right)}} 

\newcommand{\norm}[1]{\mathop{\left\lVert #1 \right\rVert}}
\newcommand{\cbrace}[1]{\mathop{\left\{ #1 \right\}}}

\newcommand{\argsS}[2]{\mathop{\left( #1 \right)#2}} 

\newcommand{\normS}[2]{\mathop{\left\lVert #1 \right\rVert#2}}

\renewcommand{\S}[1]{{\mathcal{#1}}}           	

{%
\end{array}\right.}

\usepackage{booktabs}
\begin{document}

\title{Revisiting Inaccuracies of Time Series Averaging under Dynamic Time Warping}
\author{Brijnesh Jain \\
 Technische Universit\"at Berlin, Germany\\
 e-mail: brijnesh.jain@gmail.com}
 
\date{}
\maketitle

\begin{abstract} 
This article revisits an analysis on inaccuracies of time series averaging under dynamic time warping conducted by \cite{Niennattrakul2007}. The authors presented a correctness-criterion and introduced drift-outs of averages from clusters. They claimed that averages are inaccurate if they are incorrect or drift-outs. Furthermore, they conjectured that such inaccuracies are caused by the lack of triangle inequality. We show that a rectified version of the correctness-criterion is unsatisfiable and that the concept of drift-out is geometrically and operationally inconclusive. Satisfying the triangle inequality is insufficient to achieve correctness and unnecessary to overcome the drift-out phenomenon. We place the concept of drift-out on a principled basis and show that sample means as global minimizers of a Fr\'echet function never drift out. The adjusted drift-out is a way to test to which extent an approximation is coherent. Empirical results show that solutions obtained by the state-of-the-art methods SSG and DBA are incoherent approximations of a sample mean in over a third of all trials.
\end{abstract}

\section{Introduction}

Time series such as stock prices, weather data, biomedical measurements, and biometrics data are sequences of time-dependent observations. Comparing time series is a fundamental task in various data mining applications \cite{Aghabozorgi2015,Esling2012,Fu2011}. One challenge in comparing time series is to eliminate their temporal differences \cite{Sakoe1978}. A common and widely applied technique to deal with such temporal variation is the \emph{dynamic time warping} (dtw) distance \cite{Mueller2007}. The dtw-distance has been applied to diverse data mining problems such as speech recognition \cite{Myers1981}, gesture recognition \cite{Alon2009,Reyes2011}, electrocardiogram frame classification \cite{Huang2002}, and alignment of gene expressions \cite{Aach2001}. 

Since the 1970ies, one research direction is devoted to the problem of time series averaging \cite{Abdulla2003,Cuturi2017,Gupta1996,Hautamaki2008,Lummis1973,Marteau2016,Petitjean2011,Rabiner1979,Schultz2018}. A common technique to average time series is based on first aligning the time series with respect to the dtw-distance and then synthesizing the aligned time series to an average. Several variations of this approach have been applied to improve nearest neighbor classifiers and to formulate centroid-based clustering algorithms in dtw-spaces \cite{Abdulla2003,Hautamaki2008,Morel2018,Oates1999,Petitjean2016,Rabiner1979,Soheily-Khah2016}. 

While numerous heuristics have been devised, a theoretical understanding of time series averaging is still in its early stages. Few exceptions are results on the concept of sample mean. A sample mean is any time series that minimizes the sum of squared dtw-distances from the sample time series. Examples of important theoretical results are NP-hardness of computing a sample mean \cite{Bulteau2018}, sufficient conditions of existence \cite{Jain2016b}, necessary conditions of optimality \cite{Schultz2018},  and convergence of the DTW Barycenter Averaging (DBA) algorithm  after a finite number of iterations \cite{Schultz2018}.\footnote{The DBA algorithm is an important state-of-the-art method proposed by \cite{Hautamaki2008,Petitjean2011}. In \cite{Petitjean2011}, it has been shown that DBA is a decreasing algorithm, which is not sufficient for convergence. Convergence after a finite number of iterations has been proved by \cite{Schultz2018}.} 

\medskip

In this article, we discuss two phenomena concerning inaccuracies of time series averaging raised by \cite{Niennattrakul2007}: (i) correctness of averages between two time series, and (ii)  drift-out of averages from their clusters. The authors referred to an averaging algorithm as inaccurate if it either violates their proposed correctness-criterion or suffers from drift outs. In experiments, \cite{Niennattrakul2007} observed that the state-of-the-art averaging methods of that time (2007) were inaccurate. The authors intended to make a first attempt in pointing out some misunderstanding and misuse of current DTW averaging methods at that time. They conjectured that the undesirable phenomena were caused by the lack of the triangle inequality of the dtw-distance. Consequently, they concluded that correctness of averaging in dtw-spaces cannot be guaranteed. The phenomena observed by \cite{Niennattrakul2007} and their conclusion have been cited by several publications, either as a general problem of time series averaging or as a limitation of a certain class of averaging method that can be resolved by other techniques.

The goal of this contribution is to dispel an apparently persisting misconception about the geometry of dtw-spaces. We present the following insights:
\begin{enumerate}
\itemsep0em
\item We show that there is no concept of average that satisfies (a rectified version of) the correctness-criterion (i). In addition, we argue that existence of a triangle inequality is not sufficient for satisfying criterion (i).
\item The concept of drift-out (ii) is inconclusive from an operational as well as geometrical perspective. Furthermore, existence of a triangle inequality is not necessary to ensure that an average does not drift out from its cluster. We prove that a special version of drift-outs based on the concept of sample mean never drifts out. The result can be used -- at least to a limited extent -- to assess whether an approximation of an unknown sample mean is sufficiently accurate.
\item Empirical results show that averages obtained by state-of-the-art methods such as DBA \cite{Hautamaki2008,Petitjean2011} and SSG \cite{Schultz2018} also drift out of a cluster in the sense of \cite{Niennattrakul2007}. This finding suggests that the problem of \enquote{drift-out} of previous methods still persist. However, in the new interpretation, such drift-outs merely mean that both, SSG and DBA are inaccurate approximations of the unknown sample mean. 
\end{enumerate}
The rest of this article is structured as follows: Section 2 introduces background material, Section 3 discusses the correctness-criterion and Section 4 drift-outs. Finally, Section 5 concludes with a summary of the main findings.

\section{Time Series Averaging}

\subsection{Dynamic Time Warping}
We write $[n] = \cbrace{1, \ldots, n}$ for $n \in \N$. A real-valued \emph{time series} is a sequence $x = (x_1, \ldots, x_n)$ with elements $x_i \in \R$ for all $i \in [n]$. We denote the length of $x$ by $\abs{x}$ and the set of all real-valued time series of finite length by $\S{T}$. 

A \emph{warping path} of order $m \times n$ and length $\ell$ is a sequence $p = (p_1 , \dots, p_\ell)$ consisting of $\ell$ points $p_l = (i_l,j_l) \in [m] \times [n]$ such that
\begin{enumerate}
\item $p_1 = (1,1)$ and $p_\ell = (m,n)$ \hfill (\emph{boundary conditions})
\item $p_{l+1} - p_{l} \in \cbrace{(1,0), (0,1), (1,1)}$ for all $l \in [\ell-1]$ \hfill(\emph{step condition})
\end{enumerate}
We denote the set of all warping paths of order $m \times n$ by $\S{P}_{m,n}$. 

A warping path of order $m \times n$ can be visualized as a path in a $[m] \times [n]$ grid, where rows are ordered top-down and columns are ordered left-right. The boundary conditions demand that the path starts at the upper left corner and ends in the lower right corner of the grid. The step condition demands that a transition from one point to the next point moves a unit in either down, right, or diagonal direction. A warping path $p = (p_1, \ldots, p_\ell)\in \S{P}_{m,n}$ defines a warping between two time series  $x = (x_1, \ldots, x_m)$ and $y = (y_1, \ldots, y_n)$ by matching element $x_{i_l}$ with element $y_{j_l}$ whenever $p_l = (i_l,j_l)$ is a point of $p$.

The \emph{cost} of warping time series $x$ and $y$ along warping path $p$ is defined by
\begin{equation*}
C_p(x, y) = \sum_{(i,j) \in p} \argsS{x_i-y_j}{^2},
\end{equation*}
Then the \emph{dtw-distance} of $x$ and $y$ is of the form
\begin{align*}
\dtw(x, y) = \min \cbrace{\sqrt{C_p(x, y)} \,:\, p \in \S{P}_{m,n}}.
\end{align*}
A warping path $p$ with $C_p(x, y) = \dtw^2(x, y)$ is called an \emph{optimal warping path} of $x$ and $y$. 

The dtw-distance is not a metric, because it violates the identity of indiscernibles and the triangle inequality. Instead, the dtw-distance satisfies the following properties for all $x, y \in \S{T}$: (i) $\dtw(x, y) \geq 0$, (ii) $\dtw(x, x) = 0$, and (iii) $d(x, y) = d(y, x)$. Computing the dtw-distance and deriving an optimal warping path is usually solved by applying techniques from dynamic programming \cite{Sakoe1978}. 

\subsection{Sample Mean of Time Series}\label{subsec:sample-mean}

Different forms of time series averages have been proposed. For an overview we refer to \cite{Schultz2018}. A principled formulation is based on the notion of Fr\'echet function \cite{Frechet1948}: Suppose that $\S{S} = \args{x_1, \dots, x_n}$ is a sample of $n$ time series $x_i \in \S{T}$. Then the \emph{Fr\'echet function} of $\S{S}$ is defined by
\[
F: \S{T} \rightarrow \R, \quad z \mapsto \sum_{i=1}^n \dtw\!\argsS{x_i, z}{^2}.
\]
The \emph{sample mean set} is the set  
\[
\S{F} = \cbrace{\mu \in \S{T} \,:\, \mu \in \argmin_{z \in \S{T}} F(z) }
\]
of all global minimizers of $F$. We call an element of $\S{F}$ a \emph{(sample) mean} of $\S{S}$. A mean of a sample of time series always exists but is not unique in general \cite{Jain2016b}. 

Computing a mean of a sample of time series is NP-hard \cite{Bulteau2018}. Efficient heuristics to approximate a mean of a fixed and pre-specified length are the stochastic subgradient (SSG) method \cite{Schultz2018}, soft-dtw \cite{Cuturi2017}, and a majorize-minimize algorithm \cite{Hautamaki2008,Petitjean2011} that has been popularized by \cite{Petitjean2011} under the name DTW Barycenter Averaging (DBA) algorithm.

\section{Correctness of Time Series Averaging}

\subsection{The Correctness-Criterion}

In \cite{Niennattrakul2007}, correctness of an average $\mu$ of two time series $x$ and $y$ is based on the assumption that an average \enquote{\emph{should equally contain characteristics from both original time series}}. Formally, an average $\mu$ of $x$ and $y$ is correct if 
\begin{align}\label{eq:dtw-surface}
\delta(x, \mu) = \delta(\mu, y).
\end{align}
Data mining algorithms such as k-means that use averages violating Eq.~\eqref{eq:dtw-surface} are considered as likely being incorrect \cite{Niennattrakul2007}. In experiments, \cite{Niennattrakul2007} used an equivalent formulation of Eq.~\eqref{eq:dtw-surface} to show that the averaging algorithm proposed by \cite{Gupta1996} works incorrectly.

\subsection{A Geometric Perspective on Correctness}

To discuss the correctness-criterion manifested in Eq.~\eqref{eq:dtw-surface}, it is instructive to consider averages in Euclidean spaces. The Euclidean version of Eq.~\eqref{eq:dtw-surface} takes the form
\begin{align}\label{eq:hyperplane}
\norm{x - \mu} = \norm{y - \mu}, 
\end{align}
where $x, y, \mu \in \R^n$ and $\norm{\cdot}$ is the Euclidean norm. The set of all points $\mu \in \R^n$ satisfying Eq.~\eqref{eq:hyperplane} forms a hyperplane $\S{H}_{x,y}$ that is perpendicular to the line segment $\overline{xy}$ and includes the midpoint of $\overline{xy}$. 

The hyperplane $\S{H}_{x,y}$ contains points arbitrarily remote from $x$ and $y$. Such points resist the common understanding of an average as a measure of central location. In addition,  points on $\S{H}_{x,y}$ remote from $x$ and $y$  are unsuitable as centroids for data mining applications such as k-means clustering. The same argument holds for time series $\mu$ that satisfy Eq.~\eqref{eq:dtw-surface} but are arbitrarily remote from the original time series $x$ and $y$. 

To make sense of the correctness criterion, \cite{Niennattrakul2007} demands that the time series $\mu$ in Eq.~\eqref{eq:dtw-surface} is an average of $x$ and $y$, where the concept of average is a measure of central location (see e.g.~\cite{Niennattrakul2007}, Section 3.3). Translating this assumption to Euclidean spaces gives $\mu = (x+y)/2$. In this case, we obtain 
\begin{align}\label{eq:midpoint}
\norm{x - \mu} = \norm{y - \mu} = \frac{1}{2}\norm{x - y}.
\end{align}
A point $\mu$ satisfying Eq.~\eqref{eq:midpoint} is a \emph{midpoint} of $x$ and $y$. In Euclidean spaces, a midpoint always exists and is unique. As we will see below, a midpoint $\mu$ of $x$ and $y$ also minimizes the Fr\'echet function \cite{Frechet1948}
\begin{align*}
F(z) = \frac{1}{2}\args{\normS{x-z}{^2}+\normS{y-z}{^2}}.
\end{align*}
As in Section \ref{subsec:sample-mean} a global minimizer $\mu_*$ of $F(z)$ is called a mean of $x$ and $y$. A mean $\mu_*$ has the distinguishing property that $F(\mu_*)$ is the variance of $x$ and $y$. A mean exists and is unique, because $F(z)$ is strictly convex. We obtain the mean by setting the gradient of $F(z)$ to zero and solving the equation. In doing so, we arrive at the average $\bar{\mu} = (x+y)/2$. Thus, in Euclidean spaces the concepts of midpoint, mean, and average coincide, that is $\mu = \mu_* = \bar{\mu}$. These relationships are invalid in arbitrary distance spaces, in particular in dtw-spaces as we will see next.  

\medskip

Inspired by the geometry of Euclidean spaces, we constrain Eq.~\eqref{eq:dtw-surface} by the midpoint-property in order to exclude non-central time series $\mu$ that are far from times series $x$ and $y$. A \emph{midpoint} of time series $x$ and $y$ is any time series $\mu$ satisfying  
\begin{align*}
\delta(x, \mu) = \delta(\mu, y) = \frac{1}{2}\delta(x, y).
\end{align*}
Moreover, a mean $\mu_*$ of time series $x$ and $y$ is any minimizer of the Fr\'echet function 
\begin{align*}
F(z) = \frac{1}{2}\args{\delta\argsS{x, z}{^2}+\delta\argsS{y, z}{^2}}
\end{align*}
for all $z \in \S{T}$. Finally, an \emph{average} $\bar{\mu}$ of two time series $x$ and $y$ is any time series obtained by some time series averaging method. As in \cite{Niennattrakul2007}, the definition of average is kept vague to cover the various methods proposed for time series averaging. In this sense, a sample mean is a special case of an average.

In contrast to Euclidean spaces, the concepts of midpoint and mean do not coincide in DTW-spaces. As stated in Section \ref{subsec:sample-mean}, a mean of two time series always exists but is not necessarily unique. A midpoint, however, does not exist in general. Existence of midpoints is guaranteed only in geodesic spaces \cite{Bridson1999}, which are a special subclass of metric spaces. This shows that the dtw-space is not a geodesic space, because the dtw-distance is not a metric. Moreover, since there are metric spaces that are not geodesic, existence of the triangular inequality is not sufficient for existence of a midpoint as suggested by \cite{Niennattrakul2007}. Even if a midpoint of two time series exists, it may not coincide with a mean of both time series. More generally, a mean of two time series may even violate the less constrained Eq.~\eqref{eq:dtw-surface}. 

\commentout{
We conclude this section with a remark on the statistical perspective on correctness of an average. In statistics there is no concept of average (measure of central location) that is considered as correct. Common examples of measures of central location are the arithmetic mean, median, mode, and geometric mean to mention a few, where the arithmetic mean is the epitome of all measures of central location. The arithmetic mean of two points coincides with the midpoint property. All other measures of central locations generally violate the midpoint property. This, however, does not imply that those measures are incorrect.

In dtw-spaces, common existing measures of central location are the sample mean \cite{Cuturi2017,Petitjean2011,Schultz2018}, sample median \cite{Jain2016b}, symmetric means \cite{Kruskal1983,Schultz2018}, and medoids. The correctness-criterion yields just another measure of central location not related to any of the other measures. From a statistical perspective, there is no rationale to qualify one criterion as correct and all others as incorrect. 
}

\subsection{Numerical Simulation}\label{subsec:exp01}

In this numerical simulation, we study to which extent sample means deviate from the correctness-criterion.

\medskip

\begin{table}[t]
\footnotesize
\centering
\begin{tabular}{l@{\qquad}rrl}
\toprule
Data Set & \multicolumn{1}{c}{\#} & \multicolumn{1}{c}{$n$} & Type\\ 
\midrule
ItalyPowerDemand & 1096 & 24 & SENSOR\\ 
SyntheticControl & 600 & 60 & SIMULATED\\ 
SonyAIBORobotSurface2 & 980 & 65 & SENSOR\\ 
SonyAIBORobotSurface1 & 621 & 70 & SENSOR\\ 
ProximalPhalanxTW & 605 & 80 & IMAGE\\ 
ProximalPhalanxOutlineCorrect & 891 & 80 & IMAGE\\ 
ProximalPhalanxOutlineAgeGroup & 605 & 80 & IMAGE\\ 
PhalangesOutlinesCorrect & 2658 & 80 & IMAGE\\ 
MiddlePhalanxTW & 553 & 80 & IMAGE\\ 
MiddlePhalanxOutlineCorrect & 891 & 80 & IMAGE\\ 
MiddlePhalanxOutlineAgeGroup & 554 & 80 & IMAGE\\ 
DistalPhalanxTW & 539 & 80 & IMAGE\\ 
DistalPhalanxOutlineCorrect & 876 & 80 & IMAGE\\ 
DistalPhalanxOutlineAgeGroup & 539 & 80 & IMAGE\\ 
TwoLeadECG & 1162 & 82 & ECG\\ 
\bottomrule
\end{tabular}
\caption{List of~$15$ UCR time series data sets. Columns~\# and~$n$ show the number and length of time series, respectively. 
The last column refers to the respective application domains.}
\label{tab:ucr}
\end{table}

For the simulation, we used the $15$ UCR datasets~\cite{Chen2015} listed in Table \ref{tab:ucr}. Every dataset consists of time series of identical length and comes with a pre-defined training and test split. We only considered datasets with short time series, because computing a sample mean is NP-hard \cite{Bulteau2018}. We merged the training and test sets of every dataset. 

\medskip

For every dataset, we randomly sampled $100$ pairs $(x, y)$ of time series. For every pair, we first computed a sample mean $\mu$ using the dynamic program proposed by \cite{Brill2018}. Then we calculated the dtw-distances $\dtw(x, y)$, $\dtw(x, \mu)$, and $\dtw(\mu, y)$. As in \cite{Niennattrakul2007}, the error percentage
\[
\text{err}_{\text{eq}}(\mu) = 100\frac{\abs{\dtw(x, \mu) - \dtw(\mu, y)}}{\max \cbrace{\dtw(x, \mu), \dtw(\mu, y)}}
\]
measures to which extent the sample mean $\mu$ violates the property of equidistance from $x$ and $y$. The error percentage
\[
\text{err}_{\text{mid}}(\mu) = 100\frac{\abs{\dtw(x, y) - \dtw(x, \mu) - \dtw(\mu, y)}}{\dtw(x, y)}
\]
measures to which extent the sample mean $\mu$ violates the midpoint-property.  

\medskip 

Table \ref{tab:result:midpoint-property} summarizes the error percentages of the sample means. The average error percentages $\text{err}_{\text{eq}}$ and $\text{err}_{\text{mid}}$ over $1,500$ trials are $6.29 \%$ ($\pm 5.2$) and $15.84 \%$ ($\pm 10.09$), respectively. The error percentages $\text{err}_{\text{eq}}$ ($\text{err}_{\text{mid}}$) is zero in $5$ ($1$) trials out of $1,500$ trials. These results suggest that sample means of two time series typically violate the equidistance and midpoint property. This finding indicates that the concepts of mean and midpoint typically differ in DTW-spaces and coincidences can occur exceptionally, which is in stark contrast to Euclidean spaces.

\begin{table}[h]
\footnotesize
\centering
\begin{tabular}{lc@{\qquad}rrr@{\qquad}rrr}
\toprule
dataset && \multicolumn{3}{c}{$\text{err}_{\text{eq}}$}& \multicolumn{3}{c}{$\text{err}_{\text{mid}}$} \\ 
&$n_{\text{eq}}|n_{\text{mid}}$& avg & std & max & avg & std & max \\
\midrule
ItalyPowerDemand  & $2|1$ & 6.39 & 5.85 & 27.63 & 10.79 & 6.80 & 31.32 \\
SyntheticControl      && 6.19 & 5.76 & 27.38 & 18.98 & 12.92 & 52.87 \\
SonyAIBORobotSurface2 && 5.49 & 3.89 & 17.23 & 14.60 & 4.89 & 27.30 \\
SonyAIBORobotSurface1 && 4.93 & 4.78 & 25.48 & 12.30 & 6.38 & 31.65 \\
ProximalPhalanxTW &$1|0$ & 7.19 & 5.80 & 18.57 & 15.26 & 11.07 & 29.08 \\
ProximalPhalanxOutlineCorrect & $1|0$& 6.48 & 5.24 & 18.47 & 15.43 & 10.45 & 30.16 \\
ProximalPhalanxOutlineAgeGroup & $1|0$& 7.34 & 5.49 & 18.76 & 16.42 & 10.71 & 28.25 \\
PhalangesOutlinesCorrect && 6.94 & 5.27 & 19.76 & 17.92 & 8.88 & 36.59 \\
MiddlePhalanxTW && 5.99 & 4.35 & 15.98 & 15.69 & 10.00 & 28.35 \\
MiddlePhalanxOutlineCorrect && 6.32 & 4.61 & 17.07 & 15.54 & 9.67 & 30.89 \\
MiddlePhalanxOutlineAgeGroup && 5.82 & 4.46 & 17.96 & 16.29 & 9.84 & 41.15 \\
DistalPhalanxTW && 6.42 & 5.21 & 18.92 & 15.73 & 11.14 & 37.58 \\
DistalPhalanxOutlineCorrect && 7.82 & 5.65 & 23.06 & 18.80 & 10.16 & 38.58 \\
DistalPhalanxOutlineCorrect && 6.33 & 5.38 & 19.14 & 16.00 & 11.39 & 32.62 \\
TwoLeadECG && 4.70 & 5.22 & 31.59 & 17.83 & 11.09 & 49.03 \\
\midrule
\textbf{total} &$5|1$& 6.29 & 5.20 & 31.59 & 15.84 & 10.09 & 52.87 \\
\bottomrule
\end{tabular}
\caption{Average error percentages $\text{err}_{\text{eq}}$ and $\text{err}_{\text{mid}}$, standard deviations, and maximum error percentages of sample means. The average is taken over $100$ trials for every dataset. The numbers $n_{\text{eq}}|n_{\text{mid}}$ count how often $\text{err}_{\text{eq}}(\mu) = 0$ and $\text{err}_{\text{mid}}(\mu) = 0$, resp., occurred. If no numbers are given, then the respective error percentages are always larger than zero.}
\label{tab:result:midpoint-property}
\end{table}

\section{Drift-Out Phenomenon}

\subsection{Central Regions and Drift-Outs}

According to \cite{Niennattrakul2007}, drift-outs are averages not in central location of a cluster. The original definition apparently followed the intuition shaped by Euclidean geometry. A similar depiction as in Fig.~\ref{fig:ex-drift-out-motivation} was used by \cite{Niennattrakul2007} to illustrate drift-outs. Here, we present a slightly more general approach.

\begin{definition}\label{def:drift-out-01}
Let $\S{S} = \cbrace{x_1, x_2, x_3} \subseteq \S{T}$ be a cluster of three time series and let $\S{R} = \cbrace{\mu_1, \mu_2, \mu_3}  \subseteq \S{T}$ be a set of reference time series. The $\S{R}$-\emph{central region} of cluster $\S{S}$ is the set 
\[
\S{C_{S,R}} = \cbrace{x \in \S{T} \,:\, \delta(x, x_k) \leq \delta(\mu_k, x_k) \text{ for all } k \in \cbrace{1,2,3}}.
\] 
A time series $x \in \S{T}$ is said to be \emph{drifted-out} of $\S{S}$ with respect to $\S{R}$ if $x$ is not contained in $\S{C_{S,R}}$.
\end{definition}

We call the conditions $\delta(x, x_k) \leq \delta(\mu_k, x_k)$ the \emph{centrality conditions}. In \cite{Niennattrakul2007}, the references $\mu_k \in \S{R}$ are averages of the two time series from $\S{S} \setminus\cbrace{x_k}$ for every $k \in \cbrace{1,2,3}$. Since averages in \cite{Niennattrakul2007} are solutions of an arbitrary averaging method, any time series can be an average. For this reason, Definition \ref{def:drift-out-01} imposes no restrictions on the reference set $\S{R}$.

\definecolor{Green}{rgb}{0, 0.5, 0}
\begin{figure}[t]
\centering
\includegraphics[width=0.9\textwidth]{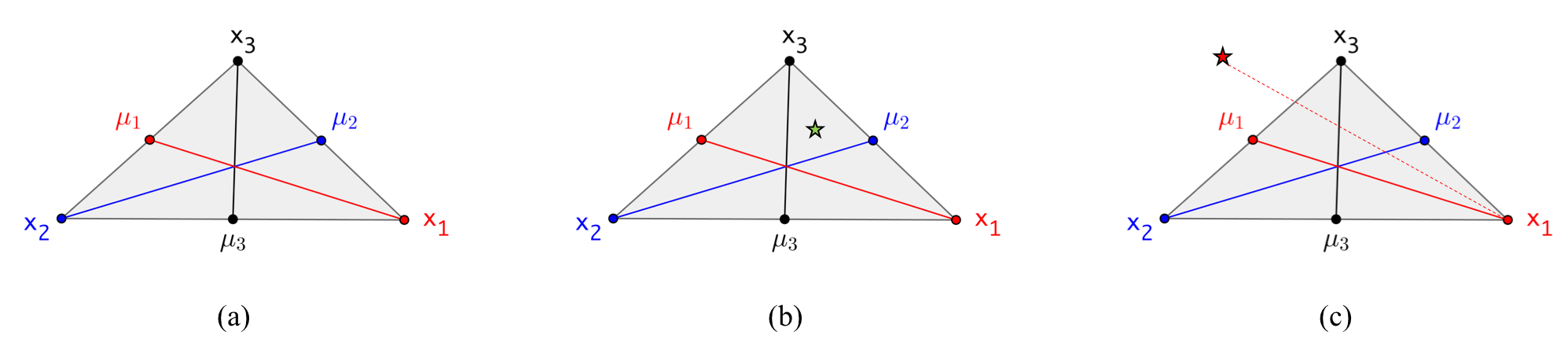} 
\caption{Illustration of drift-outs and central elements in the Euclidean plane. Plot (a) shows a triangle with three vertices $x_1, x_2, x_3$ and the midpoints $\mu_1, \mu_2, \mu_3$ of their respective opposite sides as references. Plot (b) shows a point (green star) inside the triangle satisfying the centrality conditions $\norm{{\color{Green}{\star}} - x_k} \leq \norm{\mu_k - x_k}$ for all $k = \cbrace{1,2,3}$. Plot (c) shows a point (red star) drifted out of the triangle. The dashed line segment $\overline{x_1 {\color{red}{\star}}}$ is longer than the red line segment $\overline{x_1 \mu_1}$. Thus, at least one centrality condition is violated.}
\label{fig:ex-drift-out-motivation}
\end{figure}

\subsection{An Operational Perspective on Drift-Outs}
In experiments, \cite{Niennattrakul2007} observed that the averaging method proposed by \cite{Gupta1996} suffers from the drift-out phenomenon. They state that an averaging method whose solutions drift out is incorrect.

To test whether an averaging algorithm $A$ is correct, \cite{Niennattrakul2007} applied algorithm $A$ to compute the reference averages $\mu_k \in \S{R}$ of the two time series $\S{S} \setminus\cbrace{x_k}$ and the average $\mu_{\S{S}}$ of the cluster $\S{S}$. Then the centrality conditions were checked. If at least one centrality condition was violated, the average $\mu_{\S{S}}$ was regarded as drifted out of cluster $\S{S}$.  

This approach of testing drift-outs is inconclusive. To see this, we consider the two scenarios depicted in Figure \ref{fig:ex-err-drifts}. In the first scenario, the average $\mu_{\S{S}}$ of cluster $\S{S}$ is correct with respect to some measure. However, at least one reference $\mu_k$ is severely distorted such that $\delta(\mu_{\S{S}}, x_k) \gg \delta(\mu_k, x_k)$. In this scenario, a perfect average $\mu_{\S{S}}$ is considered as drifted-out of $\S{S}$ due to a faulty reference. For the second scenario, we assume that an average $\mu_{\S{S}}$ and all references $\mu_k$ are severely distorted such that the centrality conditions are satisfied although $\mu_{\S{S}}$ is by no means central. In this case, a drift-out is falsely claimed as central to the data.

\begin{figure}[t]
\centering
\includegraphics[width=0.95\textwidth]{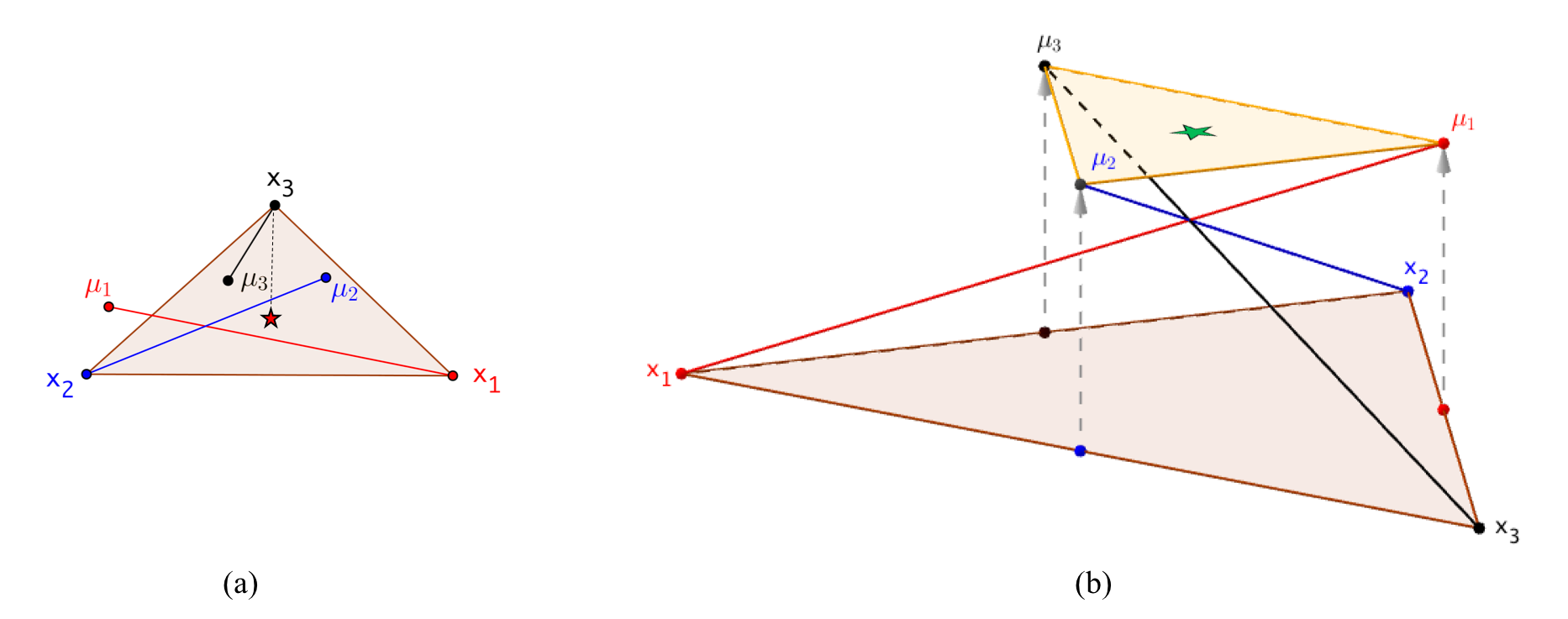} 
\caption{Illustration of a false drift-out (a) and a false central element (b). In plot (a) the average $\mu_{\S{S}}$ shown by a red star is the arithmetic mean of the vertices $x_1, x_2, x_3$. The references $\mu_1, \mu_2, \mu_3$ can be interpreted as faulty approximations of the midpoints of the triangle's sides. The point $\mu_{\S{S}}$ violates the centrality condition with respect to vertex $x_3$ and reference $\mu_3$. Thus, $\mu_{\S{S}}$ is recognized as a drift-out though it is located in a central region of the triangle. In plot (b) the average $\mu_{\S{S}}$ shown by a green star and the references $\mu_1, \mu_2, \mu_3$ are obtained by shifting the average of the three vertices and their pairwise midpoints orthogonal to the plane defined by the triangle. The average $\mu_{\S{S}}$ is a drift-out that is falsely recognized as a central element because all references drifted out of the triangle in the same way as $\mu_{\S{S}}$.}
\label{fig:ex-err-drifts}
\end{figure}

Another limitation of Definition \ref{def:drift-out-01} is that comparing the drift-out rates of two averaging algorithms is not straight-forward. From the foregoing discussion follows that drift-out rates are not comparable if every averaging algorithm constructs its own reference set. Consequently, it is necessary to find common reference sets that are suitable to ensure a fair comparison of different averaging algorithms.

To study drift-out phenomena in a more consistent way, we suggest a reference set that is independent of the choice of averaging algorithm. 
\begin{definition}
Let $\S{S} = \cbrace{x_1, x_2, x_3} \subseteq \S{T}$. A \emph{mean-based reference set} is a set $\S{R} = \cbrace{\mu_1, \mu_2, \mu_3}  \subseteq \S{T}$ consisting of sample means $\mu_k$ of $\S{S} \setminus \cbrace{x_k}$ for all $k \in \cbrace{1,2,3}$. 
\end{definition}

\medskip

Mean-based reference sets suffer from two drawbacks: (i) NP-hardness of computing a sample mean \cite{Bulteau2018}, and (ii) non-uniqueness of a sample mean.

With regard to the first limitation, currently, the most efficient approach to compute a sample mean is a dynamic program with complexity $\S{O}(k 2^k n^{2k+1})$, where $k$ is the number of time series to be averaged and $n$ is the length of the longest time series. Thus, the complexity of computing a sample mean of two time series is $\S{O}(8 n^5)$ and of three time series is $\S{O}(24 n^{7})$. Using a contemporary laptop\footnote{MacBook Pro (Retina, 15-inch, Late 2013), 2.3 GHz Intel Core i7, 16 GB 1600 MHz DDR3.}, computing sample means of two time series of length $100$ and of three time series of length $23$ both take approximately a minute. This indicates that mean-based references $\mu_k$ can be computed for time series of moderate length, whereas sample means $\mu_{\S{S}}$ of three time series can be computed only for the smallest time series.  

With regard to the second limitation, non-uniqueness of a sample mean can result in different mean-based reference sets, each of which can determine different central regions. The latter implies that there are points that drift out for one reference set but are contained in the central region defined by another reference set. Empirical evidence on the UCR benchmark datasets \cite{Chen2015} suggests that non-uniqueness of a sample mean occurs exceptionally \cite{Brill2018}.

Given a mean-based reference set, the next result shows that sample means can never drift-out.

\begin{proposition}\label{prop:centroid}
Let $\mu_{\S{S}}$ be a sample mean of cluster $\S{S} = \cbrace{x_1, x_2, x_3}\subseteq \S{T}$. Then $\mu_{\S{S}}$ is contained in the $\S{R}$-central region of $\S{S}$ for every mean-based representation set $\S{R}$.
\end{proposition}

\begin{proof}
Let $\S{S} = \cbrace{x_1, x_2, x_3} \subseteq \S{T}$ and let $\S{R} = \cbrace{\mu_1, \mu_2, \mu_3}  \subseteq \S{T}$ be an arbitrarily selected mean-based reference set, where the references $\mu_k$ are the sample means of $\S{S} \setminus \cbrace{x_k}$ for all $k \in \cbrace{1,2,3}$. Suppose that $\mu_{\S{S}}$ is a mean of $\S{S}$. Then $\mu_{\S{S}}$ minimizes the Fr\'echet function 
\[
F_{\S{S}}(z) = \sum_{i=1}^3\delta(x_i, z)^2.
\]
Let $k\in \cbrace{1,2,3}$ and let $\S{S}_k = \S{S}\setminus\cbrace{x_k} = \cbrace{x_i, x_j}$. As a sample mean of $\S{S}_k$, the time series $\mu_k$ is a global minimizer of the Fr\'echet function 
\[
F_k(z) = \delta(x_i, z)^2 + \delta(x_j, z)^2.
\]
We have 
\begin{align*}
\delta(x_k, \mu_{\S{S}})^2 = F(\mu_{\S{S}}) - F_k(\mu_{\S{S}})
\stackrel{(1)}{\leq} F(\mu_{\S{S}}) - F_k(\mu_k)
\stackrel{(2)}{\leq} F(\mu_k) - F_k(\mu_k)
= \delta(x_k, \mu_k)^2.
\end{align*}
The first inequality follows, because $\mu_{\S{S}}$ is not necessarily a minimizer of $F_k$ giving $F_k(\mu_{\S{S}}) \geq F_k(\mu_k)$. The second inequality follows, because $\mu_k$ is not necessarily a minimizer of $F$ giving $F(\mu_k) \geq F(\mu_{\S{S}})$.  Since $k$ has been arbitrarily chosen, we find that $\mu_{\S{S}} \in \S{C_{S,R}}$. The assertion follows, because the mean-based reference set has been chosen arbitrarily. 
\end{proof}

\medskip

Recall that computing a sample mean $\mu_{\S{S}}$ of a set $\S{S}$ of three time series is intractable for all but the smallest time series, whereas computing mean-based references $\mu_k$ of two time series is feasible for time series of moderate length. Thus, Prop.~\ref{prop:centroid} is useful -- at least in principle -- for testing whether a solution $z$ obtained by a heuristic is a coherent approximation of a sample mean $\mu_{\S{S}}$. We say, solution $z$ is a \emph{coherent approximation} of a sample mean, if it satisfies all three centrality conditions. Otherwise, $z$ is said to be incoherent. Note that a coherent (incoherent) approximation is not necessarily close to (remote from) a sample mean because the dtw-distance fails to satisfy the triangle inequality. 

Finally, we note that Prop.~\ref{prop:centroid} implies that the observed drift-out phenomenon is not due to the lack of triangle inequality of the dtw-distance as hypothesized in \cite{Niennattrakul2007} but rather by the way averages are constructed.

\subsection{A Geometric Perspective}

The intuition behind the concept of drift-out is purely geometrical. Informally a drift-out is a point outside of a cluster. To discuss drift-outs from a geometric perspective, we again resort to the Euclidean geometry. For this, we first restate the definition $\S{R}$-central region and drift-out for Euclidean spaces. 

Suppose that $\S{S} = \cbrace{x_1, x_2, x_3} \in \R^n$ is a cluster of three distinct points in general position. The three points of $\S{S}$ form a triangle $\S{T_S}$. We obtain the triangle $\S{T_S}$ either by connecting the points of $\S{S}$ by straight lines or by taking the convex hull of $\S{S}$. A mean-based reference set $\S{R} = \cbrace{\mu_1, \mu_2, \mu_3}  \subseteq \R^n$ consists of midpoints 
\[
\mu_1 = \frac{x_2+x_3}{2}, \quad  \mu_2 = \frac{x_1+x_3}{2}, \quad \mu_3 = \frac{x_1+x_2}{2}.
\]
The $\S{R}$-central region of $\S{S}$ is the set 
\[
\S{C_{S,R}} = \cbrace{x \in \R^n \,:\, \norm{x - x_k} \leq \norm{\mu_k - x_k} \text{ for all } k \in \cbrace{1,2,3}}.
\]
The point $x \in \R^n$ drift-outs of the triangle $\S{T_S}$ if $x \notin \S{C_{S,R}}$. 

\medskip

In plane geometry, a triangle center is informally a point that is in the middle of the figure with respect to some measure. By 2018, the Encyclopedia of Triangle Centers lists more than $15,000$ triangle centers. Such a large number of triangle centers suggests that there is no \emph{correct} central region of a triangle. Consequently, there is no \emph{correct} concept of drift-out. 

Examples of central regions are the entire triangle as a convex hull of its vertices or the incircle of a triangle. These examples comply with the intuition that a central region of a triangle is somewhere inside the triangle. In contrast, Fig.~\ref{fig:ex-drift-outs} shows that $\S{R}$-central regions can include points outside and exclude points inside a triangle.

\begin{figure}[t]
\centering
\includegraphics[width=0.5\textwidth]{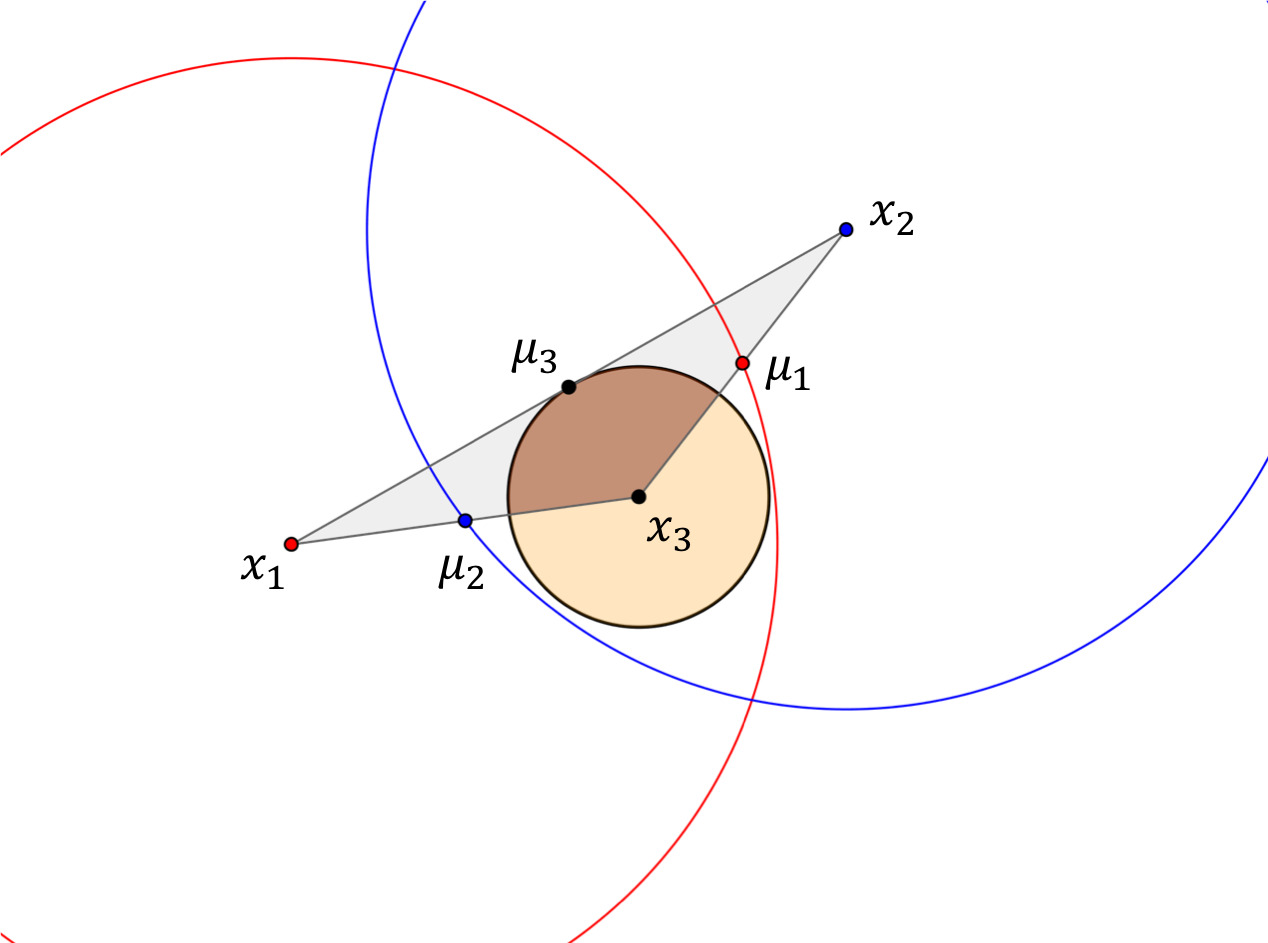} 
\caption{Example of an $\S{R}$-central region, where the references $\mu_1, \mu_2, \mu_3$ are midpoints of the triangle's sides. The $\S{R}$-central region is the intersection of the three circles with respective midpoints $x_k$ and respective radii $\norm{x_k-\mu_k}$ for all $k \in \cbrace{1,2,3}$. In this example, the intersection of the three circle is the shaded circle centered at $x_3$. The darker shading shows the part inside the triangle, whereas the lighter shading contains central points outside the triangle.}
\label{fig:ex-drift-outs}
\end{figure}

At first sight, these considerations suggest looking for other concepts of drift-outs, for example by generalizing central regions contained in a triangle to their counterpart in dtw-spaces. At second sight, this turns out to be a difficult task. As noted in the previous section, dtw-spaces are non-geodesic, meaning that the concept of \emph{straight line} between two time series is undefined in general. Without a well-defined concept of straight line, a concept of a \emph{convex set} is unknown. This implies that the concept of \emph{triangle} in dtw-spaces is unknown. Consequently, the concepts of triangle, central region, and drift-out in dtw-spaces are beyond our geometric understanding. In particular, the Euclidean view of these concepts is no longer valid in dtw-spaces.

\subsection{Numerical Simulation}

In this numerical simulation, we study how well state-of-the-art averaging methods approximate an unknown sample mean of three time series using the concept of drift-out with respect to mean-based reference sets.

\medskip

We used the same UCR datasets as in Section \ref{subsec:exp01}. We randomly sampled $100$ clusters $\S{S} = \cbrace{x, y, z}$ consisting of three time series from every UCR dataset. For every cluster $\S{S}$, we first computed sample means $\mu_i$ of $\S{S}\setminus \cbrace{x_i}$ using the dynamic program \cite{Brill2018}. Then we approximated a mean $\mu_{\S{S}}$ of the set $\S{S}$ using the DTW Barycenter Averaging (DBA) algorithm \cite{Petitjean2011} and the Stochastic Subgradient (SSG) method \cite{Schultz2018}.  We counted the number of times an approximation was drifted out of the set $\S{S}$. 

We used the following parameters for all trials: The maximum number of iterations of both mean-algorithms was set to $200$. The initial and final learning rate of SSG were set to $\eta_0 = 0.2$ and $\eta_1 = 0.02$, respectively. The learning rate was linearly decreased. No parameter-optimization for the learning rate was performed.

\begin{table}
\footnotesize
\centering
\begin{tabular}{l@{\qquad}rr}
\toprule
dataset & dba & ssg\\ 
\midrule
ItalyPowerDemand & 38.0 & 37.0 \\
SyntheticControl & 27.0 & 27.0 \\
SonyAIBORobotSurface2 & 15.0 & 15.0 \\
SonyAIBORobotSurface1 & 6.0 & 5.0 \\
ProximalPhalanxTW & 42.0 & 38.0 \\
ProximalPhalanxOutlineCorrect & 44.0 & 42.0 \\
ProximalPhalanxOutlineAgeGroup & 50.0 & 47.0 \\
PhalangesOutlinesCorrect & 47.0 & 47.0 \\
MiddlePhalanxTW & 44.0 & 44.0 \\
MiddlePhalanxOutlineCorrect & 35.0 & 33.0 \\
MiddlePhalanxOutlineAgeGroup & 58.0 & 57.0 \\
DistalPhalanxTW & 45.0 & 45.0 \\
DistalPhalanxOutlineCorrect & 41.0 & 39.0 \\
DistalPhalanxOutlineAgeGroup & 53.0 & 52.0 \\
TwoLeadECG & 21.0 & 20.0 \\
\midrule
\textbf{total} & 37.7 & 36.5\\
\bottomrule
\end{tabular}
\caption{Percentage of DBA and SSG solutions that drifted out of a set of size three.}
\label{tab:result:dropout}
\end{table}

\medskip

Table \ref{tab:result:dropout} presents the percentages of solutions obtained by DBA and SSG that drifted out of a cluster $\S{S}$ of size three. The results show that the approximations obtained by DBA and SSG are inaccurate in more than a third of all trials. In addition, inaccuracies occurred for all datasets. This result shows that neither DBA nor SSG overcomes the drift-out phenomenon observed by \cite{Niennattrakul2007}. Thus, the approximations obtained by DBA and SSG are incoherent in over a third of all cases.

\section{Discussion}

The discussion on correctness- and drift-outs shows how our intuition shaped by Euclidean geometry in two and three dimensions can lead one astray. 

The rectified correctness-criterion refers to the midpoint property and is generally unsatisfiable in non-geodesic spaces and, a fortiori, in non-metric spaces. Satisfying the triangle inequality may not be sufficient to convert the dtw-space to a geodesic space. Empirical findings on time series of moderate length indicate that a sample mean of two time series is typically neither a midpoint nor equidistant from both sample time series. 

Drift-outs refer to averages not in central location of a cluster. This conception of drift-outs is potentially misleading because it assumes an understanding about the geometry of dtw-space, we actually do not have. Since dtw-spaces are non-geodesic, concepts of straight lines, convexity, and triangles are unknown. Hence, the geometric meaning of a drift-out is unclear. Transferring the definition of drift-outs to Euclidean spaces reveals that points inside a triangle can be classified as drift-outs and points outside of a triangle as in central location. 

Under the assumption of mean-based reference sets, sample means of three time series never drift out. This result has two implications: First, the triangle inequality is not required to satisfy the centrality conditions. Second, the definition of drift-out is a way to test to which extent approximations of a sample mean are coherent. Empirical results show that solutions obtained by SSG and DBA are incoherent approximations for over a third of all trials.


\begin{thebibliography}{00}
\setlength{\parskip}{0pt}
\setlength{\itemsep}{0pt plus 0.3ex}
\small
\bibitem{Aach2001}
J.~Aach and G.M.~Church. 
\newblock Aligning gene expression time series with time warping algorithms. 
\newblock \emph{Bioinformatics}, 17(6):495–508, 2001.

\bibitem{Abdulla2003}
W.H.~Abdulla, D.~Chow, and G.~Sin. 
\newblock Cross-words reference template for DTW-based speech recognition systems. 
\newblock \emph{Conference on Convergent Technologies for Asia-Pacific Region}, 2003.

\bibitem{Alon2009}
J. Alon, V. Athitsos, Q. Yuan, and S. Sclaroff.
\newblock A unified framework for gesture recognition and spatiotemporal gesture segmentation.
\newblock \emph{IEEE Transactions on Pattern Analysis and Machine Intelligence}, 31(9): 1685–1699, 2009.

\bibitem{Aghabozorgi2015}
S.~Aghabozorgi, A.S.~Shirkhorshidi, and T.-Y.~Wah.
\newblock Time-series clustering -- A decade review.
\newblock \emph{Information Systems}, 53:16--38, 2015.

\bibitem{Bridson1999}
M.R.~Bridson and A.~Haefliger.
\newblock \emph{Metric Spaces of Non-Positive Curvature}.
\newblock Springer, 1999.

\bibitem{Brill2018}
M.~Brill, T.~Fluschnik, V.~Froese, B.~Jain, R.~Niedermeier, D.~Schultz.
\newblock Exact Mean Computation in Dynamic Time Warping Spaces
\newblock \emph{SIAM International Conference on Data Mining}, 2018.

\bibitem{Bulteau2018}
L.~Bulteau, V.~Froese, and R.Niedermeier.
\newblock Hardness of Consensus Problems for Circular Strings and Time Series Averaging.
\newblock \emph{CoRR}, abs/1804.02854, 2018.

\bibitem{Chen2015}
Y.~Chen, E.~ Keogh, B.~Hu, N.~Begum, A.~Bagnall, A.~Mueen, and G.~Batista.
\newblock \emph{The UCR Time Series Classification Archive}.
\newblock URL: \url{www.cs.ucr.edu/~eamonn/time_series_data/}, 2015.

\bibitem{Cuturi2017}
M.~Cuturi and M.~Blondel.
\newblock Soft-DTW: A Differentiable Loss Function for Time-Series.
\newblock \emph{International Conference on Machine Learning} (ICML~'17), 2017.

\bibitem{Esling2012}
P.~Esling and C.~Agon. 
\newblock Time-series data mining. 
\newblock \emph{ACM Computing Surveys}, 45(1), 2012.

\bibitem{Fu2011}
T.-C.~Fu. 
\newblock A review on time series data mining.
\newblock \emph{Engineering Applications of Artificial Intelligence}, 24(1):164--181, 2011.

\bibitem{Frechet1948}
M.~Fr\'{e}chet.
\newblock Les \'el\'ements al\'eatoires de nature quelconque dans un espace distanci\'e.
\newblock \emph{Annales de l'institut Henri Poincar\'e}, 215--310, 1948.

\bibitem{Gupta1996}
L.~Gupta, D.~Molfese, R.~Tammana, and P.G.~Simos.
\newblock Nonlinear alignment and averaging for estimating the evoked potential.
\newblock \emph{IEEE Transactions on Biomedical Engineering}, 43(4):348--356, 1996.

\bibitem{Hautamaki2008}
V.~Hautamaki, P.~Nykanen, P.~Franti.
\newblock Time-series clustering by approximate prototypes.
\newblock \emph{International Conference on Pattern Recognition}, 2008.

\bibitem{Huang2002}
B. Huang and W. Kinsner.
\newblock ECG frame classification using dynamic time warping. 
\newblock \emph{IEEE Canadian Conference on Electrical and Computer Engineering}, 2002.

\bibitem{Jain2016b}
B.J.~Jain and D.~Schultz.
\newblock A Reduction Theorem for the Sample Mean in Dynamic Time Warping Spaces.
\newblock \emph{CoRR}, abs/1610.04460, 2016.

\bibitem{Lummis1973}
R.~Lummis.
\newblock Speaker verification by computer using speech intensity for temporal registration. 
\newblock \emph{IEEE Transactions on Audio and Electroacoustics}, 21(2):80--89, 1973. 

\bibitem{Marteau2016}
P.-F.~Marteau.
\newblock Times series averaging and denoising from a probabilistic perspective on time-elastic kernels.
\newblock \emph{CoRR}, abs/1611.09194, 2016.

\bibitem{Morel2018}
M.~Morel, C.~Achard, R.~Kulpa, and S.~Dubuisson.
\newblock Time-series Averaging Using Constrained Dynamic Time Warping with Tolerance. 
\newblock \emph{Pattern Recognition}, 74, 2018.

\bibitem{Mueller2007}
M.~M\"uller.
Dynamic time warping
\emph{Information retrieval for music and motion}, 69--84, 2007.

\bibitem{Myers1981}
C.S. Myers and L.R. Rabiner. 
\newblock A comparative study of several dynamic timewarping algorithms for connected word recognition.
\newblock \emph{Bell System Technical Journal}, 60(7): 1389–1409, 1981.

\bibitem{Niennattrakul2007}
V.~Niennattrakul and C.A.~Ratanamahatana.
\newblock Inaccuracies of shape averaging method using dynamic time warping for time series data. 
\newblock \emph{International Conference on Computational Science}, 2007. 

\bibitem{Oates1999}
T.~Oates, L.~Firoiu, and P.R.~Cohen. 
\newblock Clustering Time Series with Hidden Markov Models and Dynamic Time Warping. 
\newblock IJCAI Workshop on Neural, Symbolic and Reinforcement Learning methods for Sequence Learning, 1999.

\bibitem{Petitjean2011}
F.~Petitjean, A.~Ketterlin, and P.~Gancarski. 
\newblock A global averaging method for dynamic time warping, with applications to clustering.
\newblock \emph{Pattern Recognition} 44(3):678--693, 2011.

\bibitem{Petitjean2016}
F.~Petitjean, G.~Forestier, G.I.~Webb, A.E.~Nicholson, Y.~Chen, and E.~Keogh.
\newblock Faster and more accurate classification of time series by exploiting a novel dynamic time warping averaging algorithm. 
\newblock \emph{Knowledge and Information Systems}, 47(1):1--26, 2016.

\bibitem{Rabiner1979}
L.R.~Rabiner and J.G. Wilpon.
\newblock Considerations in applying clustering techniques to speaker-independent word recognition. 
\newblock \emph{The Journal of the Acoustical Society of America}, 66(3): 663--673, 1979.

\bibitem{Reyes2011}
M. Reyes, G. Dominguez, and S. Escalera. 
\newblock Feature weighting in dynamic time warping for gesture recognition in depth data. 
\newblock \emph{IEEE International Conference on Computer Vision Workshops}, 2011.

\bibitem{Sakoe1978}
H.~Sakoe and S.~Chiba. 
\newblock Dynamic programming algorithm optimization for spoken word recognition. 
\newblock \emph{IEEE Transactions on Acoustics, Speech, and Signal Processing}, 26(1):43--49, 1978.

\bibitem{Schultz2018}
D.~Schultz and B.~Jain.
\newblock Nonsmooth analysis and subgradient methods for averaging in dynamic time warping spaces.
\newblock \emph{Pattern Recognition}, 74, 2018.

\bibitem{Soheily-Khah2016}
S.~Soheily-Khah, A.~Douzal-Chouakria, and E.~Gaussier.
\newblock Generalized k-means-based clustering for temporal data under weighted and kernel time warp.
\newblock \emph{Pattern Recognition Letters}, 75:63--69, 2016.
\end{thebibliography}
\end{document}